\documentclass{article}

\usepackage{hyperref}

\usepackage{microtype}
\usepackage{graphicx}
\usepackage{subfigure}
\usepackage{booktabs} %

\usepackage{amsmath}
\usepackage{amssymb}
\usepackage{amsthm}
\usepackage{wrapfig}
\usepackage{paralist}
\usepackage{xcolor}
\usepackage{multirow}
\usepackage[font=small,labelfont=bf]{caption}
\usepackage{empheq}

\theoremstyle{plain}
\newtheorem*{lemma}{Lemma}

\newcommand{\bw}{\mathbf{w}}
\newcommand{\bW}{\mathbf{W}}
\newcommand{\bx}{\mathbf{x}}
\newcommand{\by}{\mathbf{y}}

\newcommand{\bb}{\mathbf{b}}

\newcommand{\ba}{\mathbf{a}}
\newcommand{\bhh}{\mathbf{h}}

\newcommand{\bU}{\mathbf{U}}
\newcommand{\bV}{\mathbf{V}}

\newcommand{\bA}{\mathbf{A}}

\newcommand{\bmu}{\boldsymbol{\mu}}
\newcommand{\bSigma}{\boldsymbol{\Sigma}}

\newcommand{\btheta}{\boldsymbol{\theta}}
\newcommand{\mvec}{\mathrm{vec}}
\newcommand{\mdiag}{\mathrm{diag}}

\newcommand{\mMN}{\mathcal{MN}}

\newcommand{\spm}[1]{{\tiny \textcolor{gray}{$\pm$#1}}}

\newcommand\ktiedN{\mathop{\mbox{$k$-$\mathit{tied}$-$\mathcal{N}$}}}

\newcommand*\widefbox[1]{\fbox{\hspace{1em}#1\hspace{1em}}}

\usepackage[accepted]{icml2020}
\usepackage[T1]{fontenc}

\def\pb{{\mathbf p}}

\def\bb{{\mathbf b}}
\def\qb{{\mathbf q}}

\def\zerob{{\mathbf 0}}

\def\Ab{{\mathbf A}}

\def\Bb{{\mathbf B}}

\def\Qb{{\mathbf Q}}
\def\Db{{\mathbf D}}

\def\Mb{{\mathbf M}}

\def\Pb{{\mathbf P}}

\def\Real{{\mathbb{R}}}

\def\Ncal{\mathcal{N}}

\icmltitlerunning{The $k$-tied Normal Mean Field Posterior}

\begin{document}

\twocolumn[
\icmltitle{The $k$-tied Normal Distribution: A Compact Parameterization of~\\ Gaussian Mean Field Posteriors in Bayesian Neural Networks}

\icmlsetsymbol{atgoogle}{+}

\begin{icmlauthorlist}
\icmlauthor{Jakub Swiatkowski}{warsaw,atgoogle}
\icmlauthor{Kevin Roth}{ethz,atgoogle}
\icmlauthor{Bastiaan S. Veeling}{uva,google,atgoogle}
\icmlauthor{Linh Tran}{imperial,atgoogle}
\icmlauthor{Joshua V. Dillon}{google}
\icmlauthor{Jasper Snoek}{google}
\icmlauthor{Stephan Mandt}{uci,atgoogle}
\icmlauthor{Tim Salimans}{google}
\icmlauthor{Rodolphe Jenatton}{google}
\icmlauthor{Sebastian Nowozin}{msr,atgoogle}
\end{icmlauthorlist}

\icmlaffiliation{ethz}{ETH Zurich}
\icmlaffiliation{uva}{University of Amsterdam}
\icmlaffiliation{warsaw}{University of Warsaw}
\icmlaffiliation{imperial}{Imperial College London}
\icmlaffiliation{uci}{University of California, Irvine}
\icmlaffiliation{msr}{Microsoft Research}
\icmlaffiliation{google}{Google Research}

\icmlcorrespondingauthor{Jakub Swiatkowski}{jakub.swiatkowski@mimuw.edu.pl}

\icmlkeywords{Bayesian Deep Learning}

\vskip 0.3in
]

\newcommand{\workAtGoogle}{\textsuperscript{+}Work done while at Google }
\printAffiliationsAndNotice{\workAtGoogle}

\begin{abstract}
Variational Bayesian Inference is a popular methodology for approximating posterior distributions over Bayesian neural network weights. Recent work developing this class of methods has explored ever richer parameterizations of the approximate posterior in the hope of improving performance. In contrast, here we share a curious experimental finding that suggests instead restricting the variational distribution to a more compact parameterization. For a variety of deep Bayesian neural networks trained using Gaussian mean-field variational inference, we find that the posterior standard deviations consistently exhibit strong low-rank structure after convergence. This means that by decomposing these variational parameters into a low-rank factorization, we can make our variational approximation more compact without decreasing the models' performance. Furthermore, we find that such factorized parameterizations improve the signal-to-noise ratio of stochastic gradient estimates of the variational lower bound, resulting in faster convergence.
\end{abstract}

\section{Introduction}

Bayesian neural networks \citep{mackay1992practical, neal1993bayesian} are a popular class of deep learning models. The most widespread approach for training these models relies on variational inference~\citep{peterson1987mean,hinton1993keeping}, a training paradigm that approximates a Bayesian posterior 
with a simpler class of distributions by solving an optimization problem.  The common wisdom is that more expressive distributions lead to better posterior approximations and ultimately to better model performance. This paper puts this into question and instead finds that for Bayesian neural networks, more restrictive classes of distributions, based on low-rank factorizations, can outperform the common mean-field family.

\textbf{Bayesian Neural Networks} explicitly represent their parameter-uncertainty by forming a \emph{posterior distribution} over model parameters, instead of relying on a single point estimate for making predictions, as is done in traditional deep learning. For neural network weights $\bw$, features $\bx$ and labels $\by$, the posterior distribution $p(\bw|\bx,\by)$ is computed using Bayes' rule, which multiplies the prior distribution $p(\bw)$ and data likelihood $p(\by|\bw,\bx)$ and renormalizes. When predicting with Bayesian neural networks, we form an average over model predictions where each prediction is generated using a set of parameters that is randomly sampled from the posterior distribution. This can be viewed as a type of \emph{ensembling}, of which various types have proven highly effective in deep learning \citep[see e.g.][sec 7.11]{goodfellow2016deep}.

Besides offering improved predictive performance over single models, Bayesian ensembles are also more robust because ensemble members will tend to make different predictions on hard examples \citep{raftery2005using}. In addition, the diversity of the ensemble represents predictive uncertainty and can be used for out-of-domain detection or other risk-sensitive applications \citep{ovadia2019can}.

\textbf{Variational inference} is a popular class of methods for approximating the posterior distribution $p(\bw|\bx,\by)$, since the exact Bayes' rule is often intractable to compute for models of practical interest. This class of methods specifies a distribution $q_{\btheta}(\bw)$ of given parametric or functional form as the posterior approximation, and 
optimizes the approximation by solving an optimization problem. In particular, we minimize the Kullback-Leibler (KL) divergence $D_{KL}$ between the variational distribution $q_{\btheta}(\bw)$ and the true posterior distribution $p(\bw|\bx,\by)$, which is given by
\begin{equation} 
\begin{split} 
D_{\textrm{KL}}[q_{\btheta}(\bw) || p(\mathbf{w}|\bx,\by)]= 
  \mathbb{E}_q\left [\log \frac{q_{\btheta}(\bw)}{p(\mathbf{w}|\bx,\by)}\right ] \\= 
  \mathbb{E}_q\left [\log \frac{q_{\btheta}(\bw)}{p(\mathbf{w})p(\by|\mathbf{w},\bx) / p(\by|\bx)}\right ].
 \end{split}
 \end{equation}
Here, we do not know the normalizing constant of the exact posterior $p(\by|\bx)$, but since this term does not depend on $\mathbf{w}$, we may ignore it for the purpose of optimizing our approximation $q$. We are then left with what is called the negative Evidence Lower Bound (negative ELBO):
\begin{equation}\label{eq:elbo}
L_{q} = D_{\textrm{KL}}[q_{\btheta}(\bw) || p(\mathbf{w})] - \mathbb{E}_q[\log p(\by|\mathbf{w},\bx)].
\end{equation}
In practice, the expectation of the log-likelihood $p(\by|\mathbf{w},\bx)$ with respect to $q$ is usually not analytically tractable and instead is estimated using Monte Carlo sampling:
\begin{equation}
\begin{split} 
\mathbb{E}_{q} [\log p(\by | \mathbf{w},\bx)] 
  \approx \frac{1}{S}\sum_{s=1}^S \log p(\by | \mathbf{w}^{(s)},\bx), \\
  \quad \mathbf{w}^{(s)} \sim q_{\btheta}(\mathbf{w}),
\end{split}
\end{equation}
where the ELBO is optimized by differentiating this stochastic approximation with respect to the variational parameters $\btheta$ \citep{salimans2013fixed, kingma2013auto}.

\begin{figure}[!t]
\center{\includegraphics[width=\columnwidth]{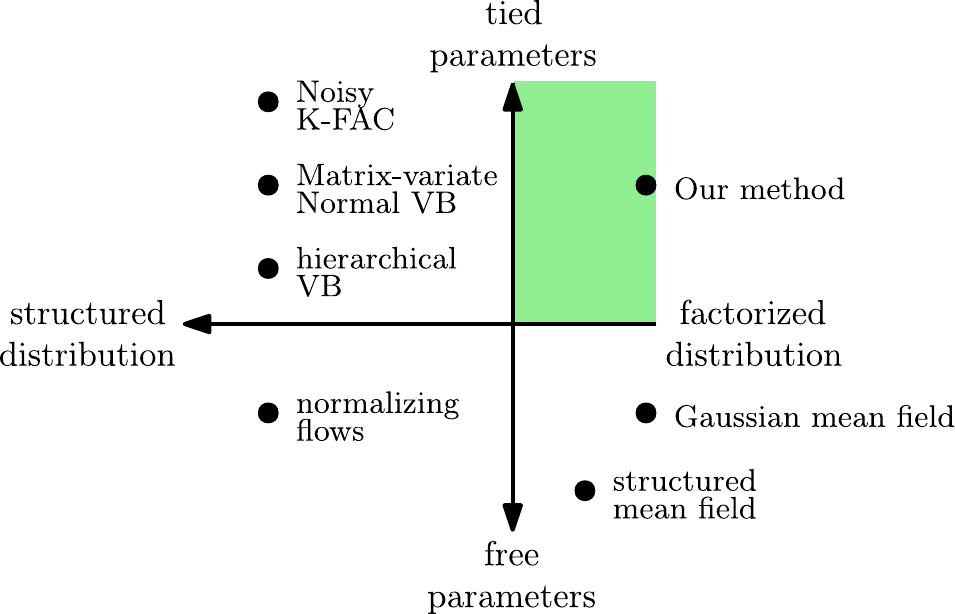}}%
\vspace{-0.2cm}%
\caption{Approximate summarization of different variational inference methods for Bayesian deep learning. Our approach complements existing approaches by combining the mean-field assumption with a dramatic reduction in the number of parameters by weight sharing.}
\vspace{-0.25cm}%
\label{fig:page1}
\end{figure}
In \textbf{Gaussian Mean Field Variational Inference} (GMFVI) ~\citep{blei2017variational,blundell2015weightuncertainty}, we choose the variational approximation to be a fully factorized Gaussian distribution $q = \mathcal{N}(\bmu_q, \bSigma_q)$ with $w_{lij} \sim \mathcal{N}(\mu_{lij}, \sigma^2_{lij})$, where $l$ is a layer number, and $i$ and $j$ are the row and column indices in the layer's weight matrix. 
 While Gaussian Mean-Field posteriors
are considered to be one of the simplest types of variational approximations, with some known
limitations~\citep{giordano2018covariances}, they scale to comparatively large models and generally provide
competitive performance~\citep{ovadia2019can}. Additionally, \citet{farquhar_try_2019} have found that the Mean-Field becomes a less restrictive assumption as the depth of the network increases. However, when compared to deterministic neural
networks, GMFVI doubles the number of parameters and is often harder to train due to the increased
noise in stochastic gradient estimates. Furthermore, despite the theoretical advantages of GMFVI over the deterministic neural networks, GMFVI suffers from over-regularization for larger networks, which leads to underfitting and often worse predictive performance in such settings~\citep{wenzel2020good}. %

Beyond mean-field variational inference, recent work on approximate Bayesian inference has explored ever richer parameterizations of the approximate posterior in the hope of improving the performance of Bayesian neural networks (see Figure~\ref{fig:page1}). 
In contrast, here we study a simpler, more compactly parameterized variational approximation. Our motivation for studying this setting is to better understand the behaviour of GMFVI with the goal to address the issues with its practical applicability. Consequently, we show that the compact approximations can also work well for a variety of models. 
In particular we find that:
\begin{compactitem}
\item Converged posterior standard deviations under GMFVI consistently display strong low-rank structure. This means that by decomposing these variational parameters into a low-rank factorization, we can make our variational approximation more compact without decreasing our model's performance.
\item Factorized parameterizations of posterior standard deviations improve the signal-to-noise ratio of stochastic gradient estimates, and thus not only reduce the number of parameters compared to standard GMFVI, but also can lead to faster convergence.
\end{compactitem}

\section{Mean Field Posterior Standard Deviations Naturally Have Low-Rank Structure}

In this section we show that the converged posterior standard deviations of Bayesian neural networks trained using standard GMFVI consistently display strong low-rank structure. We also show that it is possible to compress the learned posterior standard deviation matrix using a low-rank approximation without decreasing the network's performance. We first briefly introduce the mathematical notation for our GMFVI setting and the low-rank approximation that we explore. We then provide experimental results that support the two main claims of this section.

To avoid any confusion among the readers, we would like to clarify that we use the terminology ``low-rank'' in a particular context. While variational inference typically makes use of low-rank decompositions to compactly represent the \textit{dense covariance} of a Gaussian variational distribution (see numerous references in Section~\ref{sec:related_work}), we investigate instead underlying low-rank structures within \textit{the already diagonal covariance} of a Gaussian fully-factorized variational distribution. 
Figure~\ref{fig:ktied_normal_vs_mfvi} aims to make this even more clear by illustrating the relationship between the Gaussian fully-factorized variational distribution and its ``low-rank'' parameterization explored in this  paper. 
We will make this explanation more formal in the next section.

\begin{figure*}[!t]
\centering%
\includegraphics[width=\textwidth]{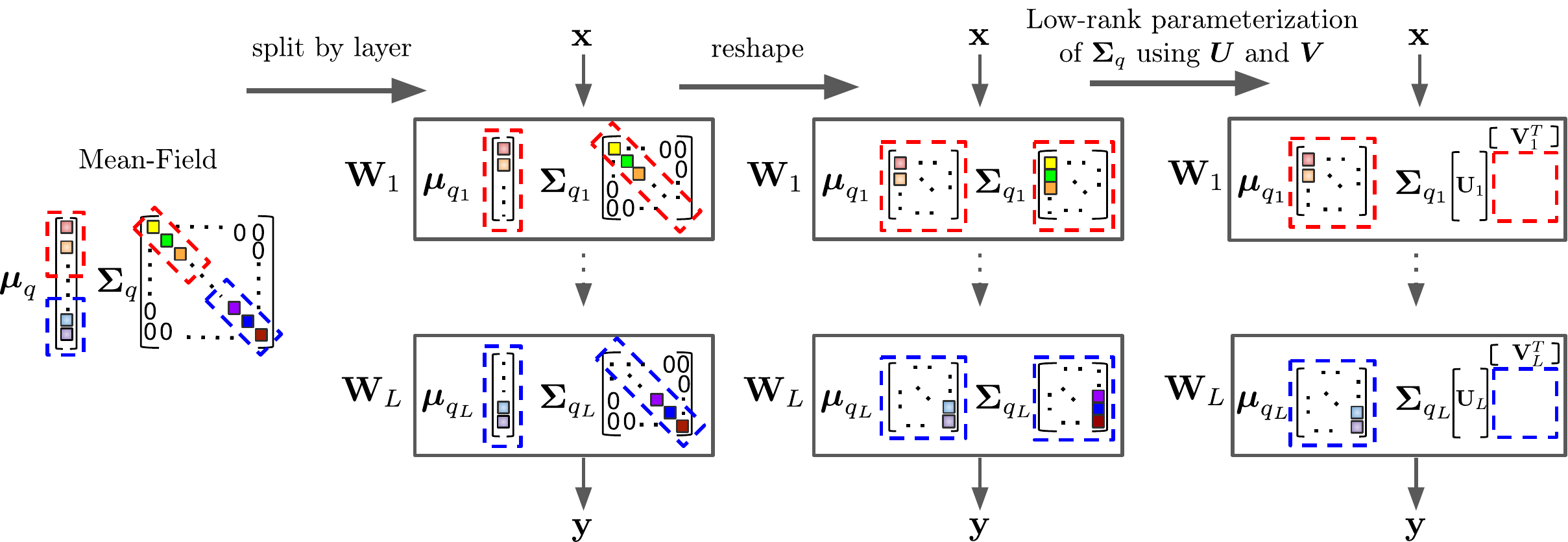}%
\caption{Illustration of the relationship between the standard Gaussian Mean-Field posterior and its ``low-rank'' parameterization, which we call the $k$-tied Normal posterior. The illustration shows the posterior parameterization for a network with $L$ layers, where $\bx$ and $\by$ are the network inputs and outputs respectively, and $\boldsymbol{\mu}_{q_1}$, $\boldsymbol{\Sigma}_{q_1}$, $\boldsymbol{\mu}_{q_L}$ and $\boldsymbol{\Sigma}_{q_L}$ are the variational parameters for the layers 1 and $L$ respectively. The $k$-tied Normal distribution parameterizes the already diagonal per layer posterior covariance matrices $\boldsymbol{\Sigma}_{q_{1..L}}$ using the even more compact $\boldsymbol{U}_{1..L}$ and $\boldsymbol{V}_{1..L}$ matrices from  $\mathcal{N}\big(\bmu_q, \mdiag\big(\mvec\big( (\mathbf{U} \mathbf{V}^T)^2 \big)\big)\big)$.}
\label{fig:ktied_normal_vs_mfvi}%
\end{figure*}

\subsection{Methodology}\label{sec:methodology}
To introduce the notation, we consider layers that consist of a linear transformation followed by a non-linearity $f$,
\begin{equation}
\mathbf{a}_l = \mathbf{h}_{l} \mathbf{W}_l + \mathbf{b}_l, \qquad \mathbf{h}_{l+1} = f(\mathbf{a}_l),
\end{equation}
where $\bW_l \in \mathbb{R}^{m \times n}$, $\bhh_{l} \in \mathbb{R}^{1 \times m}$ and $\bb_l, \ba_l, \bhh_{l+1} \in \mathbb{R}^{1 \times n}$. To simplify the notation in the following, we drop the subscript $l$ such that $\bW = \bW_l$, $\bmu_q = \bmu_{ql}$, $\bSigma_q = \bSigma_{ql}$ and we focus on the kernel matrix $\bW$ for a single layer. 

In GMFVI, we model the variational posterior as %
\begin{equation}
\begin{split}
q(\bW) = \mathcal{N}(\bmu_{q}, \bSigma_{q}) = \prod_{i=1}^m \prod_{j=1}^n q(w_{ij}), \quad \\
\mathrm{with} \quad q(w_{ij}) = \mathcal{N}(\mu_{ij}, \sigma^2_{ij}),
\end{split}
\end{equation}
where $\bmu_q \in \mathbb{R}^{mn \times 1}$ is the posterior mean vector, $\bSigma_{q} \in \mathbb{R}_{+}^{mn \times mn}$ is the diagonal posterior covariance matrix.
The weights are then usually sampled using a reparameterization trick \citep{kingma2013auto}, i.e, for the $s$-th sample, we have
\begin{equation}
w_{ij}^{(s)} = \mu_{ij} + \sigma_{ij} \epsilon^{(s)}, \qquad \epsilon \sim \mathcal{N}(0, 1).
\end{equation}
In practice, we often represent the posterior standard deviation parameters $\sigma_{ij}$ in the form of a matrix $\bA \in \mathbb{R}_{+}^{m \times n}$. Note that we have the relationship $\bSigma_{q} = \text{diag}(\text{vec}(\bA^2))$ where the elementwise-squared $\bA$ is vectorized by stacking its columns, and then expanded as a diagonal matrix into~$\mathbb{R}_{+}^{mn \times mn}$.

In the sequel, we start by empirically studying the properties of the spectrum of matrices $\Ab$ \textit{post-training} (after convergence), while using standard Gaussian mean-field variational 
distributions. Interestingly, we observe that those matrices naturally exhibit a low-rank structure (see Section~\ref{sec:post_training_low_rank} for the corresponding experiments), i.e,
\begin{equation}
\bA \approx \mathbf{U} \mathbf{V}^T
\end{equation}
for some $\bU \in \mathbb{R}^{m \times k}$, $\bV \in \mathbb{R}^{n \times k}$ and $k$ a small value (e.g., 2 or 3).
This observation motivates the introduction of the following variational family, which we name $k$-tied Normal:
\begin{empheq}[box=\widefbox]{align}
\label{eq:k_tied_normal}
\begin{split}
\ktiedN(\bW; \bmu_q, \bU, \bV) = \\  \mathcal{N}\big(\bmu_q, \mdiag\big(\mvec\big( (\mathbf{U} \mathbf{V}^T)^2 \big)\big)\big),
\end{split}
\end{empheq}
where the squaring of the matrix $\mathbf{U} \mathbf{V}^T$ is applied elementwise.
Due to the tied parameterization of the diagonal covariance matrix, we emphasize that this variational family is \textit{smaller}---i.e., included in---the standard Gaussian mean-field variational distribution family.

\begin{table}
\centering%
\resizebox{0.46\textwidth}{!}{
\begin{tabular}{lc}\hline
Variational family & Parameters (total)\\\hline
multivariate Normal & $mn + \frac{mn \, (mn+1)}{2}$\\
diagonal Normal & $mn + mn$\\
$\mathcal{MN}$ & $mn + \frac{m (m+1)}{2} + \frac{n (n+1)}{2}$\\
$\mathcal{MN}$(diagonal) & $mn + m + n$\\
$k$-tied Normal (ours) & $mn + k(m + n)$\\\hline
\end{tabular}%
}
\caption{Number of variational parameters for a variational family for a matrix $
\bW \in \mathbb{R}^{m\times n}$. $\mathcal{MN}$(diagonal) is from \citet{louizos2016structured}.}\label{tab:complexity}
\vspace{-0.25cm}%
\end{table}
As formally discussed in Appendix~\ref{appendix:proof}, the matrix variate Gaussian distribution \citep{gupta2018matrix}, referred to as $\mathcal{MN}$ and already used for variational inference by~\citet{louizos2016structured} and \citet{sun2017learning}, is related to our $k$-tied Normal distribution with $k=1$ when $\mathcal{MN}$ uses diagonal row and column covariances. Interestingly, we prove that for $k\geq 2$, our $k$-tied Normal distribution cannot be represented by any $\mMN$ distribution. This illustrates the main difference of our approach from the most closely related previous work of~\citet{louizos2016structured}. %

Notice that our diagonal covariance $\bSigma_q$ repeatedly reuses the same elements of $\bU$ and $\bV$, which results in parameter sharing across different weights. 
The total number of the standard deviation parameters in our method is $k(m + n)$ from $\bU$ and $\bV$, compared to $mn$ from $\bA$ in the standard GMFVI parameterization. Given that in our experiments the $k$ is very low (e.g. $k=2$) this reduces the number of parameters from quadratic to linear in the dimensions of the layer, see Table~\ref{tab:complexity}. More importantly, such parameter sharing across the weights leads to higher signal-to-noise ratio during training and thus in some cases faster convergence. We demonstrate this phenomena in the next section. In the rest of this section, we show that the standard GMFVI methods already learn a low-rank structure in the posterior standard deviation matrix $\bA$. Furthermore, we provide evidence that replacing $\bA$ with its low-rank approximation does not degrade the predictive performance and the quality of uncertainty estimates.

\subsection{Experimental setting}

Before describing the experimental results, we briefly explain the key properties of the experimental setting. We analyze three types of GMFVI Bayesian neural network models: 
\begin{compactitem}
\item Multilayer Perceptron (MLP): a network of 3 dense layers and ReLu activations that we train on the MNIST dataset \citep{lecun-mnisthandwrittendigit-2010}. We use the last 10,000 examples of the training set as a validation set.
\item Convolutional Neural Network (CNN): a LeNet architecture \citep{lecun1998gradient} with 2 convolutional layers and 2 dense layers that we train on the CIFAR-100 dataset \citep{krizhevsky2009learning}. We use the last 10,000 examples of the training set as a validation set.
\item Long Short-Term Memory (LSTM): a model that consists of an embedding and an LSTM cell \citep{hochreiter1997long}, followed by a single unit dense layer. We train it on the IMDB dataset \citep{maas2011learning}, in which we use the last 5,000 examples of the training set as a validation set.
\item Residual Convolutional Neural Network (ResNet): a ResNet-18\footnote{See: \url{https://github.com/tensorflow/probability/blob/master/tensorflow_probability/examples/cifar10_bnn.py}.} architecture~\citep{he2016resnet} trained on all 50,000 training examples of the CIFAR-10 dataset \citep{krizhevsky2009cifar}.
\end{compactitem}

In each of the four models, we use the standard mean-field Normal variational posterior and a Normal prior, for which we set a single scalar standard deviation hyper-parameter shared by all layers. Appendix~\ref{apx:heprior} contains an ablation study result with an alternative prior. We optimize the variational posterior parameters using the Adam optimizer \citep{kingma2014adam}. For a more comprehensive explanation of the experimental setup please refer to Appendix \ref{apx:exp_details}. Finally, we highlight that our experiments focus primarily on the comparison across a broad range of model types rather than competing with the state-of-the-art results over the specifically used datasets. %
Nevertheless, we also show that our results extend to larger models with competitive performance such as the ResNet-18 model. Note that scaling GMFVI to such larger model sizes is still a challenging research problem~\citep{osawa2019practicaldeeplearning}.

\subsection{Main experimental observation}\label{sec:post_training_low_rank}
Our main experimental observation is that the standard GMFVI learns posterior standard deviation matrices that have a low-rank structure across different model types (MLP, CNN, LSTM), model sizes (LeNet, ResNet-18) and layer types (dense, convolutional). 
To show this, we investigate the results of the SVD decomposition of posterior standard deviation matrices $\bA$ in the four described models types. We analyze the models \textit{post-training}, where the models are already trained until ELBO convergence using the standard GMFVI approach. 
While for the first three models (MLP, CNN and LSTM), we evaluate the low-rank structure only in the dense layers, for the ResNet model we consider the low-rank structure in the convolutional layers as well.

Figure~\ref{fig:exp_var} shows the fraction of variance explained per each singular value $k$ from the SVD decomposition of matrices $\bA$ in the dense layers of the first three models. The fraction of variance explained per singular value $k$ is calculated as $ \gamma_k^2 / \sum_{i^\prime}\gamma_{i^\prime}^2$,
where $\gamma$ are the singular values. %
We observe that, unlike posterior means, the posterior standard deviations have most of their variance explained by the first few singular values. In particular, a rank-1 approximation of $\bA$ explains most of its variance, while a rank-2 approximation can encompass nearly all of the remaining variance. 
Figure~\ref{fig:heatmap_stddev} %
further supports this claim visually by comparing the heat maps of the matrix $\bA$ and its rank-1 and rank-2 approximations. In particular, we observe that the rank-2 approximation heat map looks visually very similar to $\bA$, while this is not the case for the rank-1 approximation. 
Importantly, Figure \ref{fig:lowrank_approx_resnet_short} illustrates that the same low-rank structure is also present in both the dense and the convolutional layers of the larger ResNet-18 model.
In the analysis of the above experiments, we use the shorthand SEM to refer to the standard error of the mean.

\begin{figure*}[!t]
\begin{minipage}[b]{1.0\linewidth}
\centering
\includegraphics[width=170mm]{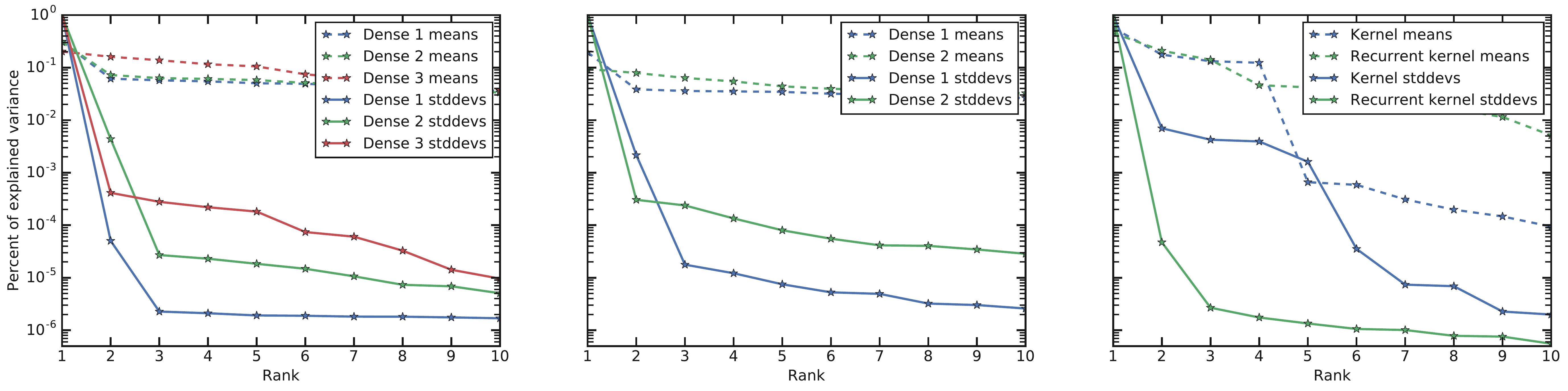}
\vspace{-0.2cm}%
\caption{Fraction of variance explained per each singular value from SVD of matrices of posterior means and posterior standard deviations post-training in different dense layers of three model types trained using standard GMFVI: MLP (left), CNN (center), LSTM (right). Unlike posterior means, posterior standard deviations clearly display strong low-rank structure, with most of the variance contained in the top few singular values.} 
\label{fig:exp_var}
\end{minipage}
\vspace{-0.5cm}
\end{figure*}

\begin{figure*}[!t] %
    \centering
    \includegraphics[width=\textwidth]{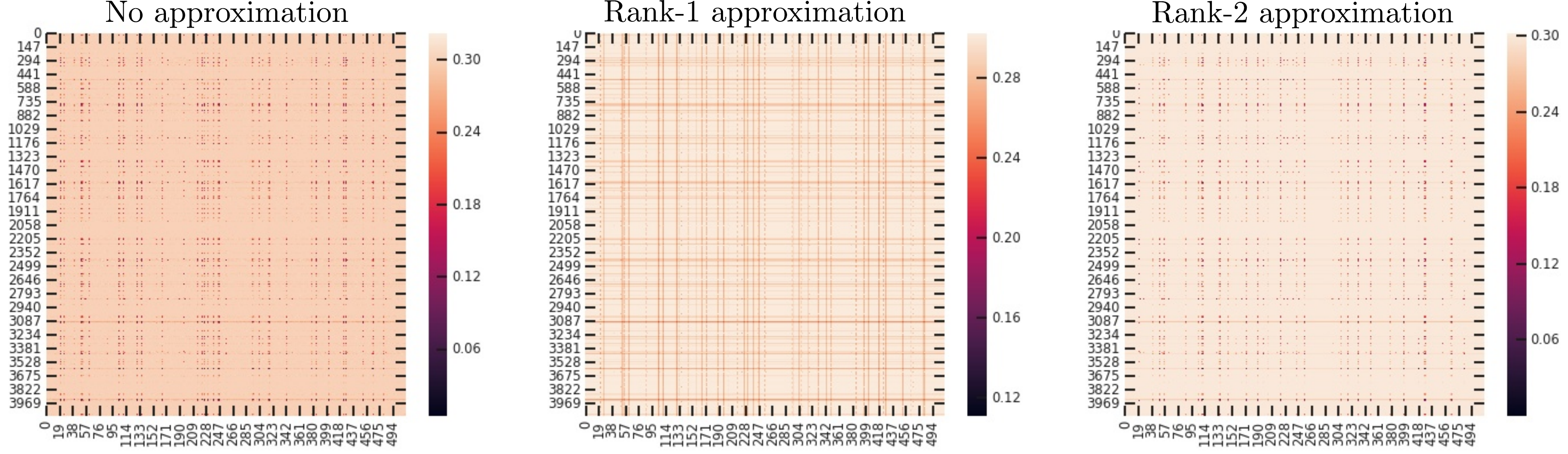}
    \captionof{figure}{Post-training heat maps of the reshaped diagonal posterior standard deviation matrix for the first dense layer of a LeNet CNN trained using GMFVI on the CIFAR-100 dataset. %
    Unlike the rank-1 approximation, the rank-2 approximation already looks visually very similar to the matrix with no approximation. This is consistent with the quantitative results from Figure~\ref{fig:exp_var}, where the first two singular values explain most of the variance in the reshaped diagonal posterior standard deviation matrix.}%
\label{fig:heatmap_stddev}%
\end{figure*}

\begin{table*}[ht]
\resizebox{\textwidth}{!}{
\begin{tabular}{lccc|ccc|ccc}
\hline
     & \multicolumn{3}{c|}{MNIST, MLP}                    & \multicolumn{3}{c|}{CIFAR100, CNN}                & \multicolumn{3}{c}{IMDB, LSTM}                      \\ \hline \hline
Method & -ELBO  $\,\downarrow$       & NLL $\,\downarrow$          & Accuracy $\,\uparrow$  & -ELBO $\,\downarrow$      & NLL $\,\downarrow$        & Accuracy $\,\uparrow$  & -ELBO   $\,\downarrow$      & NLL   $\,\downarrow$        & Accuracy   $\,\uparrow$   \\ \hline
Mean-field & 0.431\spm{0.0057} & 0.100\spm{0.0034} & 97.6\spm{0.15} & 3.83\spm{0.020} & 2.23\spm{0.017} & 42.1\spm{0.49} & 0.536\spm{0.0058} & 0.493\spm{0.0057} & 80.1\spm{0.25} \\ \hline
1-tied   & 3.41\spm{0.019}   & 0.677\spm{0.0040} & 93.6\spm{0.25} & 4.33\spm{0.021} & 2.30\spm{0.016} & 41.7\spm{0.49} & 0.687\spm{0.0058} & 0.491\spm{0.0056} & 80.0\spm{0.25} \\
2-tied    & 0.456\spm{0.0059} & 0.107\spm{0.0033} & 97.6\spm{0.15} & 3.88\spm{0.020} & 2.24\spm{0.017} & 42.2\spm{0.49} & 0.621\spm{0.0058} & 0.494\spm{0.0057} & 80.1\spm{0.25} \\
3-tied    & 0.450\spm{0.0059} & 0.106\spm{0.0033} & 97.6\spm{0.15} & 3.86\spm{0.020} & 2.24\spm{0.017} & 42.1\spm{0.49} & 0.595\spm{0.0058} & 0.493\spm{0.0056} & 80.1\spm{0.25} \\ \hline
\end{tabular}%
}%
\vspace{-0.1cm}%
\caption{Impact of post-training low-rank approximation of the GMFVI-trained posterior standard deviation matrix on model's ELBO and predictive performance, for three types of models. We report mean and SEM of each metric across 100 weights samples. The low-rank approximations with ranks higher than one achieve predictive performance close to that of mean-field without the approximations. }%
\label{tab:lowrank_approx_v1}%
\vspace{-0.2cm}%
\end{table*}

\begin{figure}[ht]
\vspace{-0.1cm}%
\begin{minipage}[b]{1.0\linewidth}
\begin{minipage}[b]{1.0\linewidth}
\centering
\includegraphics[width=65mm]{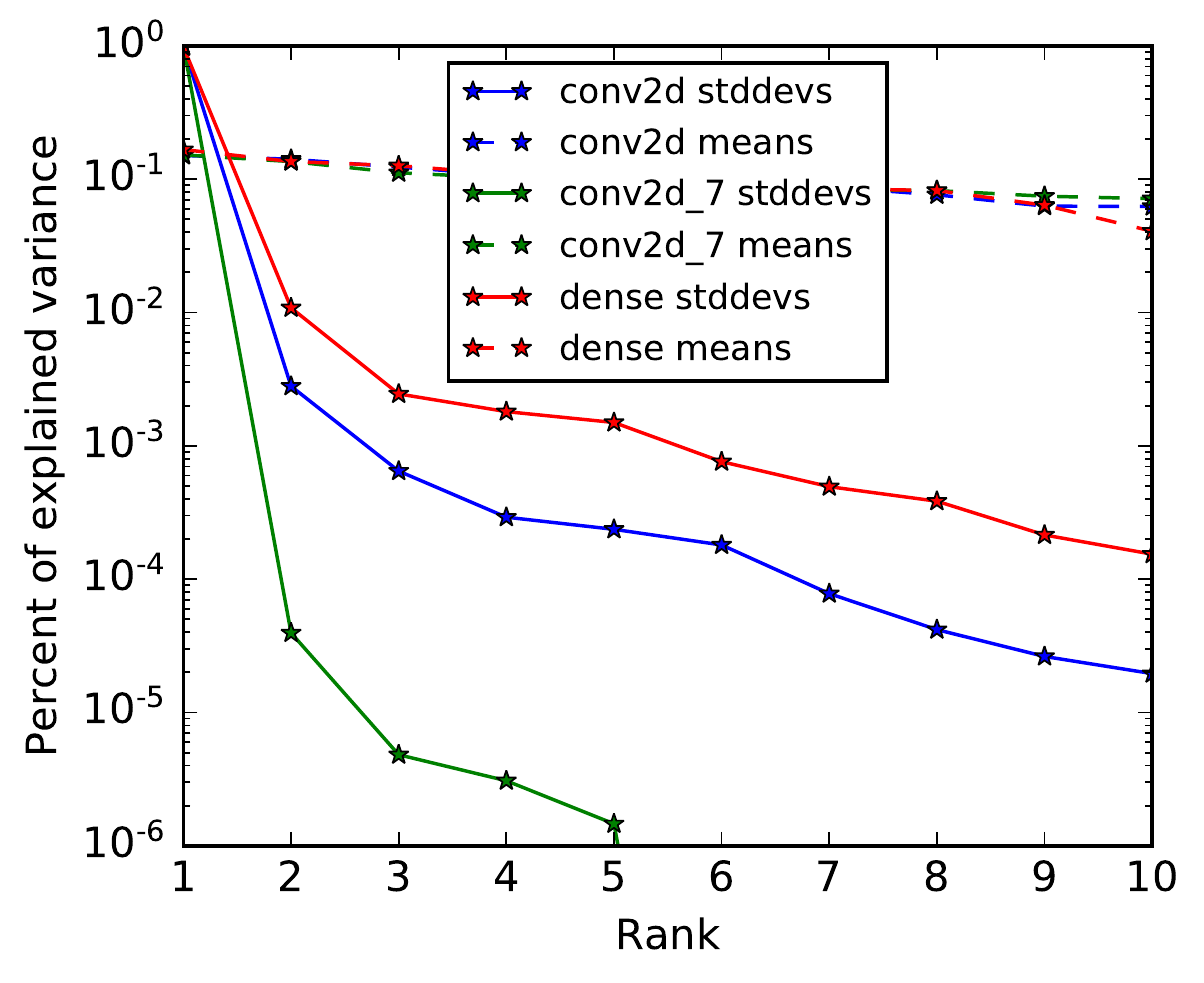}
\end{minipage} \\
\begin{minipage}[b]{1.0\linewidth}
\centering
\resizebox{0.90\textwidth}{!}{
\centering
\begin{tabular}{lccccc}
\hline
Method & -ELBO  $\,\downarrow$ & NLL $\,\downarrow$ & Accuracy $\,\uparrow$  \\ \hline \hline
Mean-field & 122.61\spm{0.012}  & 0.495\spm{0.0080} & 83.5\spm{0.37} \\ \hline
1-tied    & 122.57\spm{0.012} & 0.658\spm{0.0069} & 81.7\spm{0.39}  \\
2-tied    & 122.77\spm{0.012} & 0.503\spm{0.0080} & 83.2\spm{0.37}  \\
3-tied    & 122.67\spm{0.012} & 0.501\spm{0.0079} & 83.2\spm{0.37}  \\ \hline
\end{tabular}%
}%
\end{minipage}%
\caption{Unlike posterior means, the posterior standard deviations of both dense and convolutional layers in the ResNet-18 model trained using standard GMFVI display strong low-rank structure post-training and can be approximated without loss in predictive metrics. Top: Fraction of variance explained per each singular value of the matrices of converged posterior means and standard deviations. Bottom: Impact of post-training low-rank approximation of the posterior standard deviation matrices on the model's performance. We report mean and SEM of each metric across 100 weights samples.}
\label{fig:lowrank_approx_resnet_short}%
\end{minipage}
\vspace{-0.5cm}%
\end{figure}

\subsection{Low-rank approximation of mean field posterior standard deviations}

Motivated by the above observations, we show that it is possible to replace the reshaped diagonal posterior standard deviation matrices $\bA$ with their low-rank approximations without decreasing predictive performance and the quality of uncertainty estimates. Table \ref{tab:lowrank_approx_v1} shows the performance comparison of the MLP, CNN and LSTM models with different ranks of the approximations. %
Figure \ref{fig:lowrank_approx_resnet_short} contains analogous results for the ResNet-18 model. The results show that the post-training approximations of the mean field posterior covariance with ranks higher than one achieve predictive performance close to that of the mean field posterior with no approximations for all the analyzed model types, model sizes and layer types. 
Furthermore, Table \ref{tab:calibration} shows that, for the ResNet-18 model, the approximations with ranks higher than one also do not decrease the quality of the uncertainty estimates compared to the mean field posterior with no approximations\footnote{We compute the Brier Score and the ECE using the implementations from the TensorFlow Probability \citep{dillon2017tensorflow} Stats module: \url{https://www.tensorflow.org/probability/api_docs/python/tfp/stats}.}.
These observations could be used as a form of post-training network compression. Moreover, they give rise to further interesting exploration directions such as formulating posteriors that exploit such a low-rank structure. In the next section we explore this particular direction while focusing on the first three model types (MLP, CNN, LSTM).

\begin{table}[]
    \centering
    \resizebox{0.42\textwidth}{!}{
    \centering
    \begin{tabular}{lccccc}
    \hline
    Method & Brier Score $\,\downarrow$ & NLL $\,\downarrow$ & ECE $\,\downarrow$  \\ \hline \hline
    Mean-field & -0.761\spm{0.0039}  & 0.495\spm{0.0080} & 0.0477 \\ \hline
    1-tied    & -0.695 \spm{0.0034} & 0.658\spm{0.0069} & 0.1642  \\
    2-tied    & -0.758\spm{0.0038} & 0.503\spm{0.0080} & 0.0540  \\
    3-tied    & -0.758\spm{0.0038} & 0.501\spm{0.0079} & 0.0541  \\ \hline
    \end{tabular}%
    }
    \caption{Quality of predictive uncertainty estimates for the ResNet-18 model on the CIFAR10 dataset without and with
    post-training low-rank approximations of the GMFVI posterior standard deviation matrices in all the layers of the model. The approximations with ranks $k \geq 2$  
    match the quality of the predictive uncertainty estimates from the mean-field posteriors without the approximations. The quality of the predictive uncertainty estimates is measured by the negative log-likelihood (NLL), the Brier Score and the ECE (with 15 bins). 
    For the NLL and the Brier Score metrics we report mean and SEM across 100 weights samples.}
    \label{tab:calibration}
    \vspace{-0.5cm}%
\end{table}

\section{The $k$-tied Normal Distribution: Exploiting Low-Rank Parameter-Structure in Mean Field Posteriors}
\label{sec:ktied}

In the previous section we have shown that it is possible to replace a reshaped diagonal matrix of posterior standard deviations, which is already trained using GMFVI, with its low-rank approximation without decreasing the predictive performance. In this section, we show that it is also possible to exploit this observation during training time. We achieve this by exploiting our novel variational family, the $k$-tied Normal distribution (see Section~\ref{sec:methodology}).

We show that using this distribution in the context of GMFVI in Bayesian neural networks allows to reduce the number of network parameters, increase the signal-to-noise ratio of the stochastic gradient estimates and speed up model convergence, while maintaining the predictive performance of the standard parameterization of the GMFVI. 
We start by recalling the definition of the $k$-tied Normal distribution:
$$
\ktiedN(\bW; \bmu_q, \bU, \bV) =  \mathcal{N}\big(\bmu_q, \mdiag\big(\mvec\big( (\mathbf{U} \mathbf{V}^T)^2 \big)\big)\big)
$$
where the variational parameters are comprised of $\{\bmu_q, \bU, \bV\}$.

\subsection{Experimental setting}
We now introduce the experimental setting in which we evaluate the GMFVI variational posterior parameterized by the $k$-tied Normal distribution. We assess the impact of the described posterior in terms of predictive performance and reduction in the number of parameters for the same first three model types (MLP, CNN, LSTM) and respective datasets (MNIST, CIFAR-100, IMDB) as we used in the previous section. Additionally, we also analyze the impact of $k$-tied Normal posterior on the signal-to-noise ratio of stochastic gradient estimates of the variational lower bound for the CNN model as a representative example. 
Overall, the experimental setup is very similar to the one introduced in the previous section. Therefore, we highlight only the key differences here. 

We apply the $k$-tied Normal variational posterior distribution to the same layers which we analyzed in the previous section. 
Namely, we use the $k$-tied Normal variational posterior for all the three layers of the MLP model, the two dense layers of the CNN model and the LSTM cell's kernel and recurrent kernel. We initialize the parameters $u_{ik}$ from $\bU$ and $v_{jk}$ from $\bV$ of the $k$-tied Normal distribution so that after the outer-product operation the respective standard deviations $\sigma_{ij}$ have the same mean values as we obtain in the standard GMFVI posterior parameterization. In the experiments for this section, we use \textit{KL annealing}~\citep{sonderby2016train}, where we linearly scale-up the contribution of the $D_{\textrm{KL}}[q_{\btheta}(\bw) || p(\mathbf{w})]$ term in Equation~\ref{eq:elbo} from zero to its full contribution over the course of training. %
Appendix~\ref{apx:kladam} describes the impact of KL annealing on the modelled uncertainty.
Furthermore, additional details on the experimental setup are available in Appendix~\ref{apx:exp_details}.

\subsection{Experimental results}
We first investigate the predictive performance of the GMFVI Bayesian neural network models trained using the $k$-tied Normal posterior distribution, with different levels of tying $k$. We compare these results to those obtained from the same models, but trained using the standard parameterization of the GMFVI. Figure \ref{tab:exploit_res} (left) shows 
that for $k \geq 2$ the $k$-tied Normal posterior is able to achieve the performance competitive with the standard GMFVI posterior parameterization, while reducing the total number of model parameters. 
The benefits of using the $k$-tied Normal posterior are the most visible for models where the layers with the $k$-tied Normal posterior constitute a significant portion of the total number of the model parameters (e.g. the MLP model). %

We further investigate the impact of the $k$-tied Normal posterior on the signal-to-noise ratio (SNR)\footnote{SNR for each gradient value is calculated as $E[g_b^2] / \mathrm{Var}[g_b]$, where $g_{b}$ is the gradient value for a single parameter. The expectation $E$ and variance $Var$ of the gradient values $g_b$ are calculated over a window of last 10 batches.} of stochastic gradient estimates of the variational lower bound (ELBO). In particular, we focus on the gradient SNR of the GMFVI posterior standard deviation parameters for which we perform the tying.
These parameters are either $u_{ik}$ and $v_{jk}$ for the $k$-tied Normal posterior or $\sigma_{ij}$ for the standard GMFVI parameterization, all optimized in their log forms for numerical stability. Figure \ref{tab:exploit_res} (top right) shows that the $u_{ik}$ and $v_{jk}$ parameters used in the $k$-tied Normal posterior are trained with significantly higher gradient SNR than the $\sigma_{ij}$ parameters used in the standard GMFVI parameterization.
Consequently, Figure \ref{tab:exploit_res} (bottom right) shows that the increased SNR from the $k$-tied Normal distribution translates into faster convergence for the MLP model, which uses the $k$-tied Normal distribution in all of its layers. 

Note that the $k$-tied Normal posterior does not increase the training step time compared to the standard parameterization of the GMFVI, see Table \ref{tab:speed} %
for the support of this claim\footnote{Code to compare the training step times of the $k$-tied Normal and the standard GMFVI is available under: \url{https://colab.research.google.com/drive/14pqe_VG5s49xlcXB-Jf8S9GoTFyjv4OF}. The code uses the network architecture from: \url{https://github.com/tensorflow/docs/blob/master/site/en/tutorials/keras/classification.ipynb}.}. Therefore, the $k$-tied Normal posterior speeds up model convergence also in terms of wall-clock time.

Figure \ref{fig:exploit_results} shows the convergence plots of validation negative ELBO for all the three model types. 
We observe that the impact of the $k$-tied Normal posterior on convergence depends on the model type. As shown in Figure \ref{tab:exploit_res} (bottom right),  the impact on the MLP model is strong and consistent with the $k$-tied Normal posterior increasing convergence speed compared to the standard GMFVI parameterization. For the LSTM model we also observe a similar speed-up. However, for the CNN model the impact of the $k$-Normal posterior on the ELBO convergence is much smaller. We hypothesize that this is due to the fact that we use the $k$-tied Normal posterior for all the layers trained using GMFVI in the MLP and the LSTM models, while in the CNN model we use the $k$-tied Normal posterior only for some of the GMFVI trained layers. More precisely, in the CNN model we use the $k$-tied Normal posterior only for the two dense layers, while the two convolutional layers are trained using the standard parameterization of the GMFVI.

\begin{figure*}[ht]
\begin{minipage}[b]{1.0\linewidth}
\begin{minipage}[c]{0.63\linewidth}
\resizebox{\textwidth}{!}{
\begin{tabular}{lc|cccc}
\hline
Model \& Dataset   & Method &  -ELBO $\,\downarrow$ & NLL $\,\downarrow$ & Accuracy $\,\uparrow$ & \#Par. [k]$\,\downarrow$ \\
 \hline \hline
MNIST, MLP  & Mean-field  & 0.501\spm{0.0061}   & 0.133\spm{0.0040}   & 96.8\spm{0.18} &  957 \\ \hline
MNIST, MLP    & 1-tied & 0.539\spm{0.0063}   & 0.155\spm{0.0043}   & 96.1\spm{0.19}  & 482  \\
MNIST, MLP    & 2-tied  & 0.520\spm{0.0063}   & 0.129\spm{0.0039}   & 96.8\spm{0.18} & 484 \\
MNIST, MLP    & 3-tied  & 0.497\spm{0.0060}   & 0.120\spm{0.0038}   & 96.9\spm{0.18} & 486  \\ \hline
CIFAR100, CNN  & Mean-field & 3.72\spm{0.018}   & 2.16\spm{0.016}   & 43.9\spm{0.50} &  4,405  \\ \hline
CIFAR100, CNN     & 1-tied & 3.65\spm{0.017}   & 2.12\spm{0.015}   & 45.5\spm{0.50}  & 2,262  \\
CIFAR100, CNN     & 2-tied & 3.76\spm{0.019}   & 2.15\spm{0.016}   & 44.3\spm{0.50} & 2,268  \\ 
CIFAR100, CNN & 3-tied   & 3.73\spm{0.018}   & 2.13\spm{0.016}   & 44.3\spm{0.50} & 2,273  \\ \hline
IMDB, LSTM  & Mean-field & 0.538\spm{0.0054}   & 0.478\spm{0.0052}   & 79.5\spm{0.26} &  2,823 \\ \hline
IMDB, LSTM  & 1-tied & 0.592\spm{0.0041}   & 0.512\spm{0.0040}   & 77.6\spm{0.26}  & 2,693 \\
IMDB, LSTM & 2-tied & 0.560\spm{0.0042}   & 0.484\spm{0.0041}   & 78.2\spm{0.26} & 2,694 \\
IMDB, LSTM & 3-tied & 0.550\spm{0.0051}   & 0.491\spm{0.0050}   & 78.8\spm{0.26} & 2,695 \\ \hline
\end{tabular}%
}
\end{minipage}
\hfill
\begin{minipage}[]{0.34\linewidth}
\centering
\begin{minipage}[t]{1.00\linewidth}
\centering
\resizebox{\textwidth}{!}{
\begin{tabular}{c|ccc}
\hline
\multirow{2}{*}{Method} & \multicolumn{3}{c}{MNIST, MLP Dense 2, SNR at step} \\
 & 1000 &  5000 & 9000 \\
 \hline \hline
Mean-field  & 
 4.13\spm{0.027}    & 4.45\spm{0.091} & 3.21\spm{0.035} \\ \hline
 1-tied & 
5840\spm{190}     & 158\spm{3.8} & 5.3\spm{0.20} \\
 2-tied  & 
7500\spm{240}       & 140\spm{11} & 4.3\spm{0.26} \\
 3-tied  & 
7000\spm{270} & 117\spm{1.7}& 4.1\spm{0.20} \\ \hline
\end{tabular}%
}
\end{minipage} \\ 
\vspace{0.5cm}%
\begin{minipage}[b]{1.00\linewidth}
\centering
\resizebox{\textwidth}{!}{
\begin{tabular}{c|ccc}
\hline
 \multirow{2}{*}{Method} & \multicolumn{3}{c}{MNIST, MLP, -ELBO at step} \\
 & 1000 &  5000 & 9000 \\
 \hline \hline
 Mean-field  & 
 42.16\spm{0.070}    & 26.52\spm{0.016} & 15.39\spm{0.016} \\ \hline
 1-tied & 
43.11\spm{0.039}     & 14.85\spm{0.017} & 2.06\spm{0.027} \\
 2-tied  & 
42.74\spm{0.090}       &  13.97\spm{0.023} & 1.82\spm{0.017} \\
 3-tied  & 
42.63\spm{0.068} & 13.61\spm{0.020}& 1.80\spm{0.031} \\ \hline
\end{tabular}%
}
\end{minipage}%
\end{minipage}
\caption{Left: impact of the $k$-tied Normal posterior on test ELBO, test predictive performance and number of model parameters. We report the test metrics on the test splits of the respective datasets as a mean and SEM across 100 weights samples after training each of the models for $\approx$300 epochs. The $k$-tied Normal distribution with rank $k \geq 2$ allows to train models with smaller number of parameters without decreasing the predictive performance.
Top right: mean gradient SNR in the log posterior standard deviation parameters of the Dense 2 layer of the MNIST MLP model at increasing training steps for different ranks of tying $k$. The $k$-tied Normal distribution significantly increases the SNR for these parameters.
We observe a similar increase in the SNR from tying in all the layers that use the $k$-tied Normal posterior.
Bottom right: negative ELBO on the MNIST validation dataset at increasing training steps for different ranks of tying $k$. The higher SNR from the $k$-tied Normal posterior translates into the increased convergence speed for the MLP model. %
We report mean and SEM across 3 training runs with different random seeds in both the top right and the bottom right table. 
}
\label{tab:exploit_res}%
\end{minipage}
\end{figure*}

\begin{figure*}[ht]%
\vspace{-0.1cm}%
\begin{minipage}[b]{1.0\linewidth}%
\centering%
\includegraphics[width=170mm]{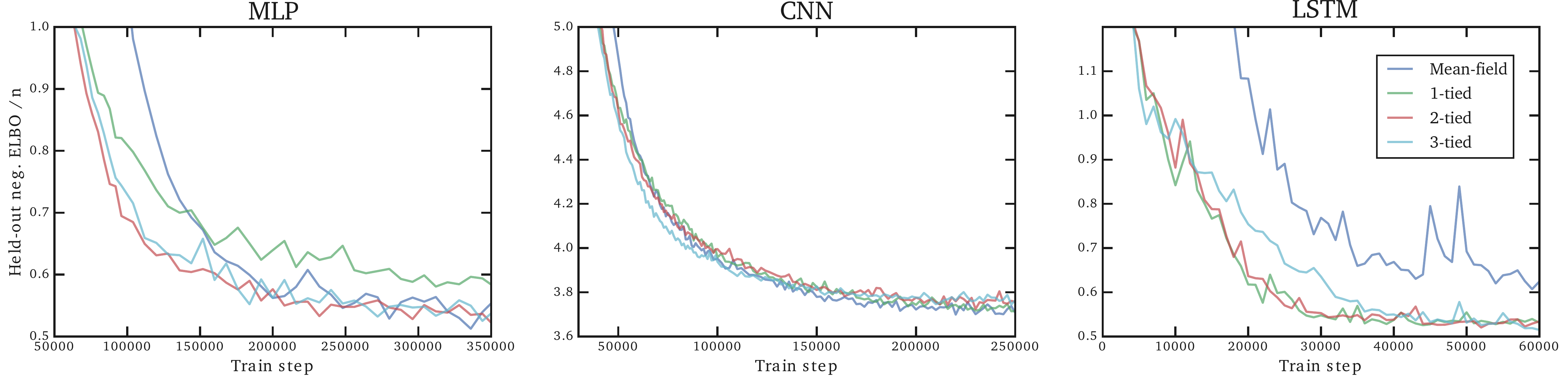}%
\end{minipage}%
\vspace{-0.2cm}%
\caption{Convergence of negative ELBO (lower is better) for the three model types on their respective validation datasets when training with mean-field and $k$-tied variational posteriors. The $k$-tied Normal posteriors result in faster initial convergence for the MLP and LSTM models. For the CNN models the speed-up is not as significant when using the $k$-tied Normal posterior only for the two dense layers of the LeNet architecture.}%
\label{fig:exploit_results}%
\vspace{-0.3cm}%
\end{figure*}

\begin{table}%
    \centering
    \resizebox{0.42\textwidth}{!}{
    \centering
    \begin{tabular}{l|c}
    \hline
    Training method               & Train step time {[}ms{]} $\,\downarrow$\\ \hline \hline
    Point estimate      & 2.00\spm{0.0064}             \\ \hline
    Standard GMFVI & 7.17\spm{0.014}              \\ \hline
    2-tied Normal GMFVI        & 6.14\spm{0.018}              \\ \hline
    \end{tabular}
    }
      \captionof{table}{Training step evaluation times for a simple model architecture with two dense layers
      for different training methods.
      We report mean and SEM of evaluation times across a single training run in the Google Colab environment linked in the footnote. The $k$-tied Normal posterior with $k = 2$ does not increase the train step evaluation times compared to the standard parameterization of the GMFVI posterior. We expect this to hold more generally because the biggest additional operation per step when using the $k$-tied Normal posterior is the $\mathbf{U} \mathbf{V}^T$
    multiplication to materialize the matrix of posterior standard deviations $\mathbf{A}$, where $\bU \in \mathbb{R}^{m \times k}$, $\bV \in \mathbb{R}^{n \times k}$ and $k$ is a small value (e.g., 2 or 3).
    The time complexity of this operations is $\mathcal{O}(kmn)$, which is usually negligible compared to the time complexity of data-weight matrix multiplication $\mathcal{O}(bmn)$, where $b$ is the batch size.
}
      \label{tab:speed}
    \vspace{-0.3cm}%
    \end{table}

\section{Related Work}\label{sec:related_work}

The application of variational inference to neural networks dates back at least to~\citet{peterson1987mean} and  \citet{hinton1993keeping}.
Many developments\footnote{We refer the interested readers to \citet{zhang2018advances} for a recent review of variational inference.} have followed those seminal research efforts, in particular regarding (1) the expressiveness of the variational posterior distribution and (2) the way the variational parameters themselves can be structured to lead to compact, easier-to-learn and scalable formulations. We organize the discussion of this section around those two aspects, with a specific focus on the Gaussian case. 

\paragraph{Full Gaussian posterior.} Because of their substantial memory and computational cost, Gaussian variational distributions with full covariance matrices have been primarily applied to (generalized) linear models and shallow neural networks~\citep{jaakkola1997variational, barber1998ensemble, marlin2011piecewise,titsias2014doubly,miller2017variational, ong2018gaussian}.

To represent the dense covariance matrix efficiently in terms of variational parameters, several schemes have been proposed, including the sum of low-rank plus diagonal matrices~\citep{barber1998ensemble, seeger2000bayesian, miller2017variational, zhang2017noisy, ong2018gaussian}, the Cholesky decomposition~\citep{challis2011concave} or by operating instead on the precision matrix~\citep{tan2018gaussian,mishkin2018slang}.

\paragraph{Gaussian posterior with block-structured covariances.}  
In the context of Bayesian neural networks, the layers represent a natural structure to be exploited by the covariance matrix. When assuming independence across layers, the resulting covariance matrix exhibits a \textit{block-diagonal structure} that has been shown to be a well-performing simplification of the dense setting~\citep{sun2017learning, zhang2017noisy}, with both memory and computational benefits.

Within each layer, the corresponding diagonal block of the covariance matrix can be represented by a Kronecker product of two smaller matrices~\citep{louizos2016structured, sun2017learning}, possibly with a parameterization based on rotation matrices~\citep{sun2017learning}.
Finally, using similar techniques, \citet{zhang2017noisy} proposed to use a block tridiagonal structure that better approximates the behavior of a dense covariance.

\paragraph{Fully factorized mean-field Gaussian posterior.}
A fully factorized Gaussian variational distribution constitutes the simplest option for variational inference. The resulting covariance matrix is diagonal and all underlying parameters are assumed to be independent. While the mean-field assumption is known to have some limitations---e.g., underestimated variance of the posterior distribution~\citep{turnersahani2011} and robustness issues~\citep{giordano2018covariances}---it leads to scalable formulations, with already competitive performance, as for instance illustrated by the recent uncertainty quantification benchmark of~\citet{ovadia2019can}.

Because of its simplicity and scalability, the fully-factorized Gaussian variational distribution has been widely used for Bayesian neural networks~\citep{graves2011practicalvb, ranganath2014black, blundell2015weightuncertainty, hernandez2015probabilisticbp,zhang2017noisy, khan2018fast}.

Our approach can be seen as an attempt to further reduce the number of parameters of the (already) diagonal covariance matrix. Closest to our approach is the work of~\citet{louizos2016structured}. Their matrix variate Gaussian distribution instantiated with the Kronecker product of the diagonal row- and column-covariance matrices leads to a rank-1 tying of the posterior variances. In contrast, we explore tying strategies beyond the rank-1 case, which we show to lead to better performance (both in terms of ELBO and predictive metrics). Importantly, we further prove that tying strategies with a rank greater than one cannot be represented in a matrix variate Gaussian distribution, thus clearly departing from~\citet{louizos2016structured} (see Appendix~\ref{appendix:proof} for details).

Our approach can be also interpreted as a form of \textit{hierarchical} variational inference from ~\citet{ranganath2016hierarchical}. In this interpretation, the prior on the variational parameters corresponds to a Dirac distribution, non-zero only when a pre-specified low-rank tying relationship holds. 
More recently, \citet{karaletsos2018probabilistic}
proposed a hierarchical structure which also couples network weights, but achieves this by introducing 
representations of network units as latent variables.

Our reduction in the number parameters through tying decreases the variance of the stochastic gradient estimates of the ELBO objective for the posterior standard deviation parameters.
Alterantive methods for the variance reduction propose to either change the noise sampling scheme~\cite{kingma2015variationaldropout, wen2018flipout, farquhar_radial_2020} or to determinize the variational posterior approximation \cite{wu2019dvi}.
Those methods can be combined with our method, but some of them are not valid in cases when our method is applicable (e.g. the method from \citet{kingma2015variationaldropout} is not valid for convolutional layers).

We close this related work section by mentioning the existence of other strategies to produce more flexible approximate posteriors, e.g., normalizing flows~\citep{rezende2015variational} and extensions thereof~\citep{louizos2017multiplicative}.

\section{Conclusion}
In this work we have shown that Bayesian Neural Networks trained with standard Gaussian Mean-Field Variational Inference learn posterior standard deviation matrices that can be approximated with little information loss by low-rank SVD decompositions. This suggests that richer parameterizations of the variational posterior may not always be needed, and that compact parameterizations can also work well. We used this insight to propose a simple, yet effective variational posterior parameterization, which speeds up training and reduces the number of variational parameters without degrading predictive performance on %
a range of model types.

In future work, we hope to scale up variational inference with compactly parameterized approximate posteriors to much larger models and more complex problems. For mean-field variational inference to work well in that setting several challenges will likely need to be addressed \citep{wenzel2020good}; improving the signal-to-noise ratio of ELBO gradients using our compact variational parameterizations may provide a piece of the puzzle.

\bibliography{icml2020}
\bibliographystyle{icml2020}

\clearpage
\appendix

\section{Proof of the Matrix Variate Normal Parameterization}\label{appendix:proof}

In this section of the appendix, we formally explain the connections between the $k$-tied Normal distribution and the matrix variate Gaussian distribution \citep{gupta2018matrix}, referred to as $\mathcal{MN}$.

Consider positive definite matrices $\Qb \in \Real^{r \times r}$ and $\Pb \in \Real^{c \times c}$ and some arbitrary matrix $\Mb \in \Real^{r \times c}$. We have by definition that $\bW \in \Real^{r \times c} \sim \mathcal{MN}(\Mb, \Qb, \Pb)$ if and only if $\text{vec}(\bW) \sim \Ncal(\text{vec}(\Mb), \Pb \otimes \Qb)$, where $\text{vec}(\cdot)$ stacks the columns of a matrix and $\otimes$ is the Kronecker product

The $\mathcal{MN}$ has already been used for variational inference by~\citet{louizos2016structured} and \citet{sun2017learning}. In particular, \citet{louizos2016structured} consider the case where both $\Pb$ and $\Qb$ are restricted to be diagonal matrices. In that case, the resulting distribution corresponds to our $k$-tied Normal distribution with $k=1$ since
$$
\Pb \otimes \Qb = \textrm{diag}(\pb) \otimes \textrm{diag}(\qb)= \textrm{diag}(\textrm{vec}(\qb \pb^\top)).
$$
Importantly, we prove below that, in the case where $k \geq 2$, the $k$-tied Normal distribution cannot be represented as a matrix variate Gaussian distribution.

\begin{lemma}[Rank-2 matrix and Kronecker product] Let $\Bb$ be a rank-$2$ matrix in $\Real_{+}^{r \times c}$. There do not exist matrices $\Qb \in \Real^{r \times r}$ and $\Pb \in \Real^{c \times c}$ such that 
$$
\textrm{diag}(\textrm{vec}(\Bb)) = \Pb \otimes \Qb.
$$
\end{lemma}
\begin{proof}
Let us introduce the shorthand $\Db = \textrm{diag}(\textrm{vec}(\Bb))$. By construction, $\Db$ is diagonal and has its diagonal terms strictly positive (it is assumed that $\Bb \in \Real_{+}^{r \times c}$, i.e., $b_{ij} > 0$ for all $i, j$).

We proceed by contradiction. Assume there exist $\Qb \in \Real^{r \times r}$ and $\Pb \in \Real^{c \times c}$ such that 
$\Db = \Pb \otimes \Qb$. 

This implies that all diagonal blocks of $\Pb \otimes \Qb$ are themselves diagonal with strictly positive diagonal terms. Thus, $p_{jj} \Qb$ is diagonal for all $j \in \{1,\dots,c\}$, which implies in turn that $\Qb$ is diagonal, with non-zero diagonal terms and $p_{jj} \neq 0$. 
Moreover, since the off-diagonal blocks $p_{ij} \Qb$ for $i \neq j$ must be zero and $\Qb \neq \zerob$, we have $p_{ij}=0$ and $\Pb$ is also diagonal.

To summarize, if there exist $\Qb \in \Real^{r \times r}$ and $\Pb \in \Real^{c \times c}$ such that 
$\Db = \Pb \otimes \Qb$, then it holds that $\Db = \textrm{diag}(\pb) \otimes \textrm{diag}(\qb)$ with $\pb \in \Real^c$ and $\qb \in \Real^r$. This last equality can be rewritten as $b_{ij} = p_j q_i$ for all $i \in \{1,\dots,r\}$ and $j \in \{1,\dots,c\}$, or equivalently
$$
\Bb = \qb \pb^\top.
$$
This leads to a contradiction since $\qb \pb^\top$ has rank one while $\Bb$ is assumed to have rank two.
\end{proof}

Figure~\ref{fig:ktied_normal_vs_mvn} provides an illustration of the difference between the $k$-tied Normal and the $\mMN$ distribution.
\begin{figure}[ht]
\begin{minipage}[b]{1.0\linewidth}%
\centering%
\includegraphics[width=68mm]{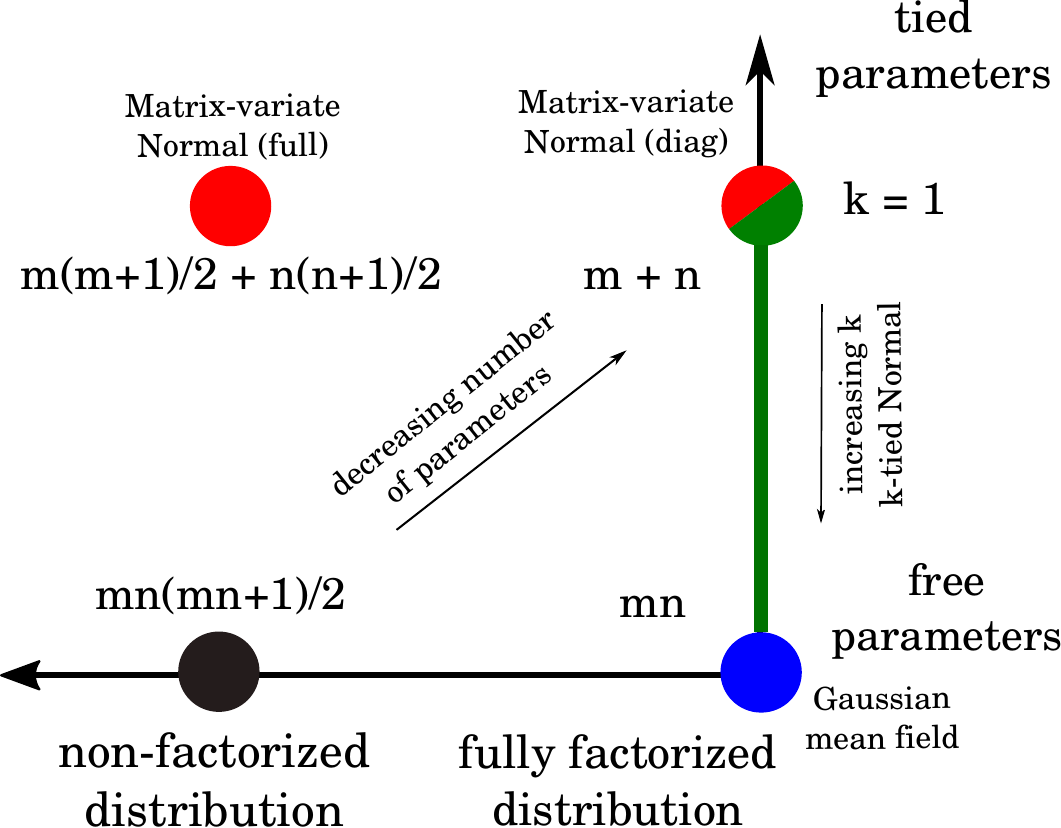}%
\end{minipage}%
\caption{Illustration of the difference in modeling of the posterior covariance by the $k$-tied Normal distribution (green), the $\mMN$ distribution (red), the Gaussian mean field (blue) and the dense Gaussian covariance (black) for a layer of size $m \times n$. The $k$-tied Normal with $k = 1$ is equivalent to $\mMN$ with diagonal row and column covariance matrices (half-red, half-green circle). Our experiments show that the $k=1$ fails to capture the performance of the mean field. On the other hand, while the full/non-diagonal $\mMN$ increases the expressiveness of the posterior, it also increases the number of parameters. In contrast, the $k$-tied Normal distribution with $k \geq 2$ not only decreases the number of parameters, but also matches the predictive performance of the mean field.}%
\label{fig:ktied_normal_vs_mvn}%
\end{figure}

\section{He-scaled Normal Prior}\label{apx:heprior}
We investigate whether the low-rank structure is specific to the GMFVI neural networks that use a Normal prior with a single scalar scale for all the weights.
Instead of using the single scale parameter, we analyse a setting in which the Normal prior scale is set according to the scaling rules devised for neural network weights initialization \cite{glorot2010understanding,he2015delving}.
According to these rules, a per layer scale parameter is set according to the layer shape and activation function used.
In particular, we use the scaling rule from \citet{he2015delving} for the models with ReLU activations~\cite{glorot2011deep}:
\begin{equation}
p(\bw_l) = \mathcal{N}\left(0, \frac{2}{m_l}\right),
\end{equation}
where $m_l$ is the \textit{fan-in} of the $m$'th layer.\footnote{For a dense layer the fan-in is the number of input dimensions, for a 2D Convolutional layer with a kernel of size $k\times k$ and $d$ input
channels the fan-in is $m_l = k^2d$.}
However, the scaling rule proposed in \citet{he2015delving} does not cover the bias terms, which are initialized at zero.
Therefore, for the ResNet-18 on CIFAR-10 which we take under test, we keep the prior for the biases unchanged at $\mathcal{N}(0,I)$.
We rerun then the low-rank structure experiments from Section~\ref{sec:post_training_low_rank} Figure~\ref{fig:lowrank_approx_resnet_short},
but now with the He-scaled prior.
Figure~\ref{fig:lowrank_approx_resnet_he} shows the low-rank structure analysis results for the new prior.
While we observe an overall drop in performance, the low-rank structure clearly remains present.

\begin{figure}[ht]
\vspace{-0.1cm}%
\begin{minipage}[b]{1.0\linewidth}
\begin{minipage}[b]{1.0\linewidth}
\centering
\includegraphics[width=75mm]{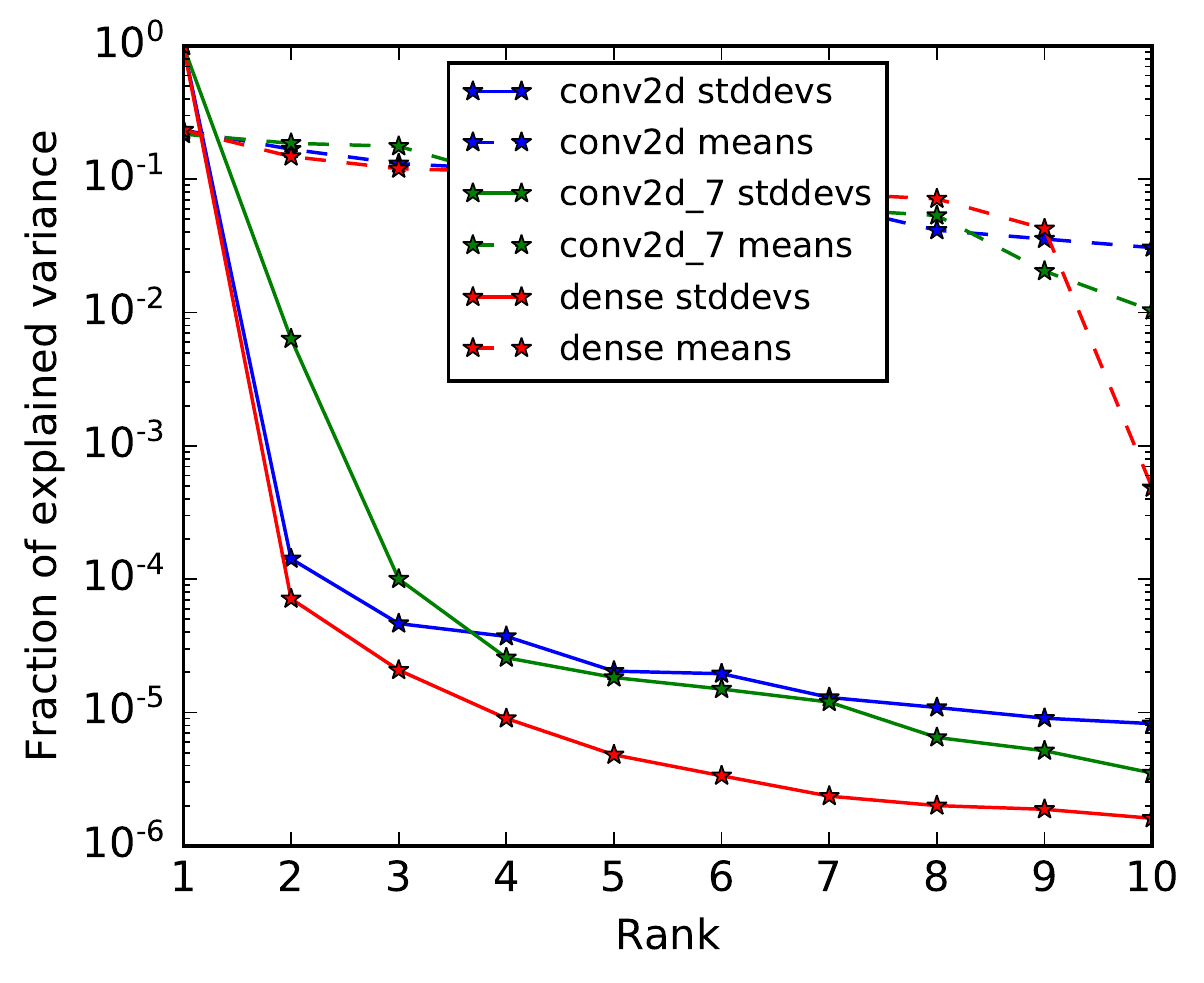}
\end{minipage} \\
\begin{minipage}[b]{0.96\linewidth}
\centering
\resizebox{1.00\textwidth}{!}{
\centering
\begin{tabular}{lccccc}
\hline
Method & -ELBO  $\,\downarrow$ & NLL $\,\downarrow$ & Accuracy $\,\uparrow$  \\ \hline \hline
Mean-field & 1.379\spm{0.0096}  & 0.6384\spm{0.0096} & 79.0\spm{0.41} \\ \hline
1-tied    & 5.428\spm{0.018} & 1.485\spm{0.0056} & 57.0\spm{0.50}  \\
2-tied    & 1.448\spm{0.0097} & 0.648\spm{0.0079} & 78.8\spm{0.41}  \\
3-tied    & 1.411\spm{0.0097} & 0.646\spm{0.0079} & 78.9\spm{0.41}  \\ \hline
\end{tabular}%
}%
\end{minipage}%
\caption{The post-training low rank structure is still present in the posterior standard deviation parameters of the
ELBO-converged standard GMFVI ResNet-18 CIFAR-10 model when using the He-scaled prior.
Approximations to these parameters with ranks higher than 1 result in performance close to that when not using the approximation.
We report mean and SEM for predictions made using an ensemble of 100 weights samples. The SEM is measured across the test examples.}
\label{fig:lowrank_approx_resnet_he}%
\end{minipage}
\vspace{-0.5cm}%
\end{figure}

\section{KL Annealing with Adam}\label{apx:kladam}

We verify that when using KL annealing with Adam the posterior standard deviation parameters do not converge prematurely,
but rather continue being optimized after the KL is at its full contribution.
Figure~\ref{fig:kl_anneal_check} illustrates this on the example of the ResNet-18 CIFAR-10 model trained the standard GMFVI.
Furthermore, for the MLP, CNN and LSTM models, we observed their posterior standard deviations at convergence to have
large values compared to the prior standard deviation value (>50\% of the prior value),
showing that we are modeling substantial uncertainty.

\begin{figure}[ht]
\vspace{-0.1cm}%
\begin{minipage}[b]{1.0\linewidth}
\begin{minipage}[b]{1.0\linewidth}
\centering
\includegraphics[width=75mm]{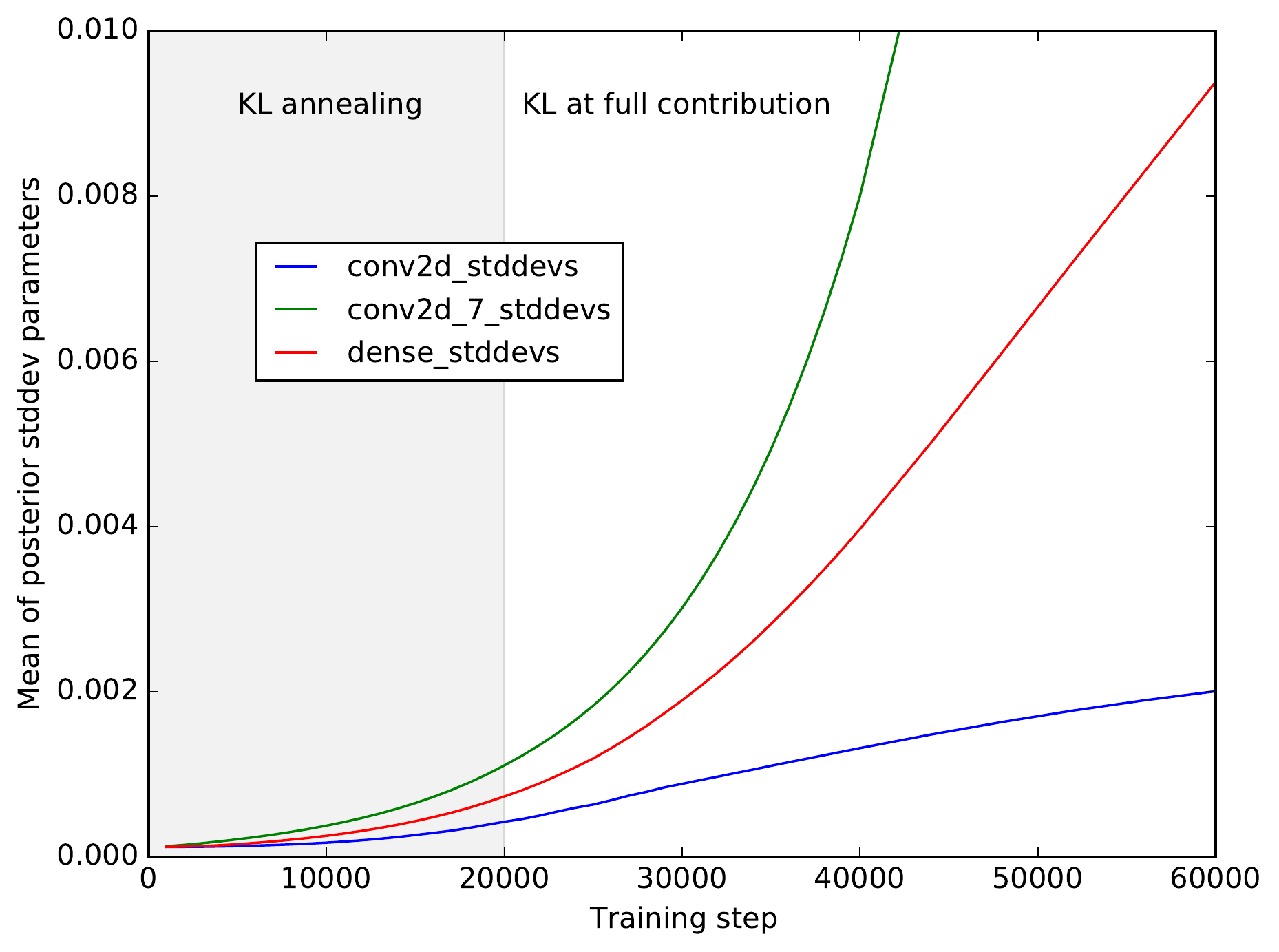}
\label{fig:kl_anneal_check}%
\end{minipage} \\
\begin{minipage}[b]{1.0\linewidth}
\centering
\includegraphics[width=75mm]{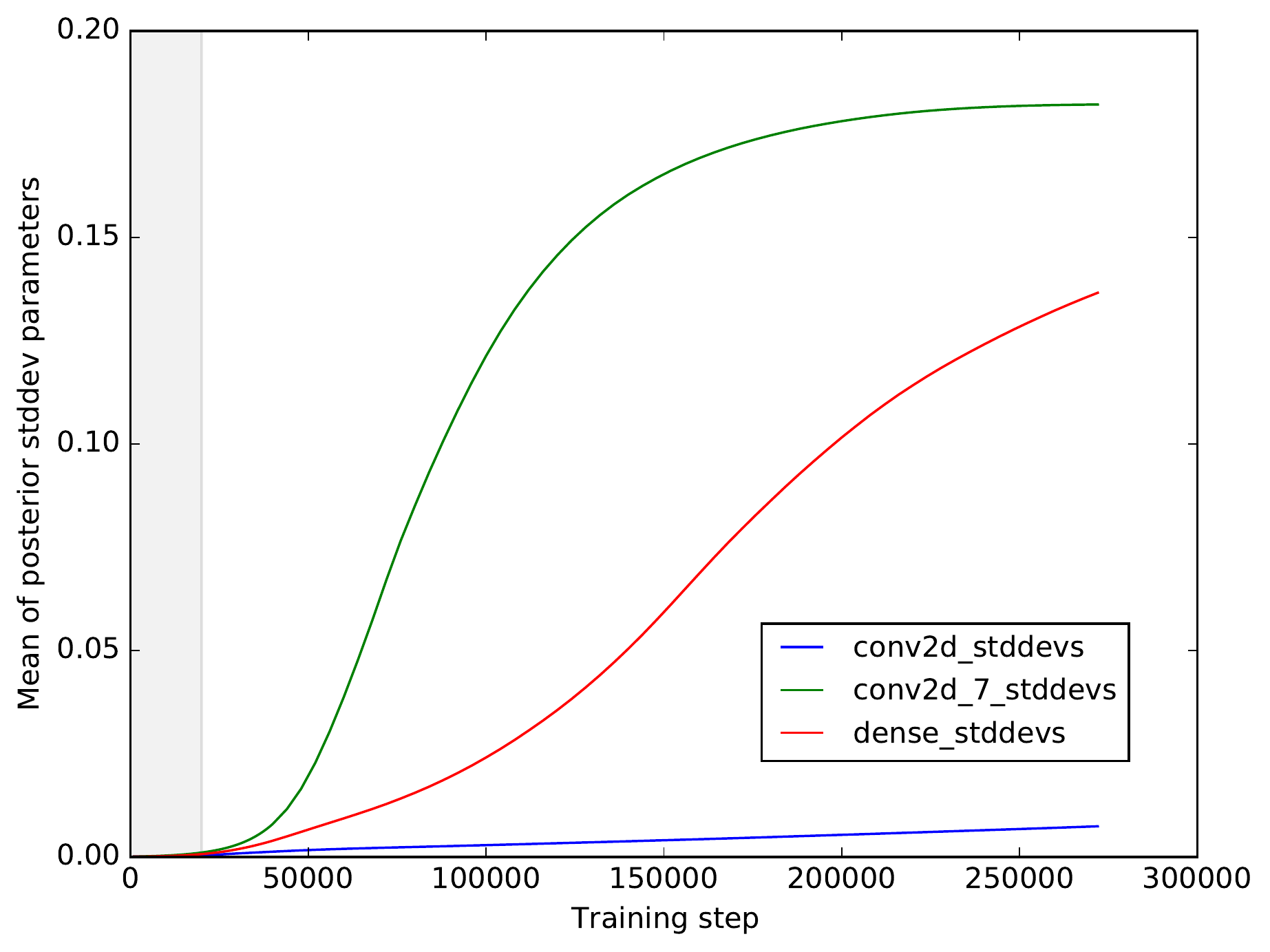}
\vspace{-0.5cm}%
\caption{Change in the mean of posterior standard deviation parameters for selected layers of the standard GMFVI ResNet-18 CIFAR-10 model over the course of training.
KL is annealed over the first 50 epochs linearly from 0 to 1 (gray area). Top: posterior standard deviation parameters continue being optimized when the KL is at its full contribution.
Bottom: the posterior standard deviations reach large values after 700 epochs showing that we are modeling substantial uncertainty.}
\label{fig:kl_anneal_check}%
\vspace{-0.5cm}%
\end{minipage}
\end{minipage}
\end{figure}

\section{Experimental Details}
\label{apx:exp_details}
In this section we provide additional information on the experimental setup used in the main paper. In particular, we describe the details of the models and datasets, the utilized standard Gaussian Mean Field Variational Inference (GMFVI) training procedure, the low-rank structure analysis of the GMFVI trained posteriors and the proposed $k$-tied Normal posterior training procedure.

\subsection{Models and datasets} 
To confirm the validity of our results, we performe the experiments on a range of models and datasets with different data types, architecture types and sizes. Below we describe their details.

\paragraph{MLP MNIST} Multilayer perceptron (MLP) model with three dense layers and ReLu activations trained on the MNIST dataset \citep{lecun-mnisthandwrittendigit-2010}. The three layers have sizes of 400, 400 and 10 hidden units. We preprocess the images to be have values in range $[-1, 1]$. We use the last 10,000 examples of the training set as a validation set. 

\paragraph{LeNet CNN CIFAR-100} LeNet convolutional neural network (CNN) model \citep{lecun1998gradient} with two convolutional layers followed by two dense layers, all interleaved with ReLu activations. The two convolutional layers have 32 and 64 output filters respectively, each produced by kernels of size $3\times3$. The two dense layers have sizes of 512 and 100 hidden units. We train this network on the CIFAR-100 dataset \citep{krizhevsky2009learning}. We preprocess the images to have values in range $[0, 1]$. We use the last 10,000 examples of the training set as a validation set. 

\paragraph{LSTM IMDB} Long short-term memory (LSTM) model \citep{hochreiter1997long} that consists of an embedding and an LSTM cell, followed by a dense layer with a single unit. The LSTM cell consists of two dense weight matrices, namely the kernel and the recurrent kernel. The embedding and the LSTM cell have both 128-dimensional output space. More precisely, we adopt the publicly available LSTM Keras~\citep{chollet2015keras} example\footnote{See:~\url{https://github.com/keras-team/keras/blob/master/examples/imdb_lstm.py}.}, except that we set the dropout rate to zero. We train this model on the IMDB text sentiment classification dataset \citep{maas2011learning}, in which we use the last 5,000 examples of the training set as a validation set. 

\paragraph{ResNet-18 CIFAR-10} ResNet-18 model \citep{he2016deep} trained on the CIFAR-10 dataset \citep{krizhevsky2009learning}. We adopt the ResNet-18 implementation\footnote{See:~\url{https://github.com/tensorflow/probability/blob/master/tensorflow_probability/examples/cifar10_bnn.py}.} 
from the Tensorflow Probability~\citep{dillon2017tensorflow} repository. We train/evaluate this model on the train/test split of 50,000 and 10,000 images, respectively, from the CIFAR-10 dataset available in Tensorflow Datasets\footnote{See:~\url{https://www.tensorflow.org/datasets/catalog/cifar10}.}.

\subsection{GMFVI training}
We train all the above models using GMFVI. We split the discussion of the details of the GMFVI training procedure into two parts. First, we describe the setup for the MLP, CNN and LSTM models, for which we prepare our own GMFVI implementations. Second, we explain the setup for the GMFVI training of the ResNet-18 model, for which we use the implementation available in the Tensorflow Probability repository as mentioned above. %

\paragraph{MLP, CNN and LSTM}

In the MLP and the CNN models, we approximate the posterior using GMFVI for all the weights (both kernel and bias weights). For the LSTM model, we approximate the posterior using GMFVI only for the kernel weights, while for the bias weights we use a point estimate. For all the three models, we use the standard reparametrization trick estimator~\citep{kingma2013auto}. We initialize the GMFVI posterior means using the standard He initialization \citep{he2015delving} and the GMFVI posterior standard deviations using samples from $\mathcal{N}(0.01, 0.001)$. Furthermore, we use a Normal prior $\mathcal{N}(0, \sigma_p I)$ with a single scalar standard deviation hyper-parameter $\sigma_p$ for all the layers. We select $\sigma_p$ for each of the models separately from a set of $\{0.2, 0.3\}$ based on the validation data set performance.

We optimize the variational parameters using an Adam optimizer \citep{kingma2014adam}. We pick the optimal learning rate for each model from the set of $\{0.0001, 0.0003, 0.001, 0.003\}$ also based on the validation data set performance. We choose the batch size of 1024 for the MLP and CNN models, and the batch size of 128 for the LSTM model. We train all the models until the ELBO convergence. %

To implement the MLP and CNN models we use the \texttt{tfp.layers} module from the Tensorflow Probability, while to implement the LSTM model we use the \texttt{LSTMCellReparameterization}\footnote{See:~\url{https://github.com/google/edward2/blob/master/edward2/tensorflow/layers/recurrent.py}.} class from the Edward2 Layers module~\cite{tran2019bayesian}.

\paragraph{ResNet-18}
The specific details of the GMFVI training of the ResNet-18 model can be found in the previously linked implementation from the Tensorflow Probability repository. Here, we describe the most important and distinctive aspects of this implementation.

The ResNet-18 model approximates the posterior using GMFVI only for the kernel weights, while for the bias weights it uses a point estimate. The model uses the Flipout estimator~\citep{wen2018flipout} and a constraint on the maximum value of the GMFVI posterior standard deviations of $0.2$. The GMFVI posterior means are initialized using samples from $N(0, 0.1)$, while the GMFVI posterior log standard deviations are initialized using samples from $\mathcal{N}(-9.0, 0.1)$. Furthermore, the model uses a Normal prior $\mathcal{N}(0, I)$ for all of its layers. 

The variational parameters are trained using the Adam optimizer with a learning rate of $0.0001$ and a batch size of 128. The model is trained for 700 epochs. The contribution of the $D_{KL}$ term in the negative Evidence Lower Bound (ELBO) equation %
is \textit{annealed} linearly from zero to its full contribution over the first 50 epochs~\citep{sonderby2016train}.

\subsection{Low-rank structure analysis}

After training the above models using GMFVI, we investigate the low-rank structure in their trained variational posteriors. For the MLP, CNN and LSTM models, we investigate the low-rank structure of their dense layers only. For the ResNet-18 model, we investigate both its dense and convolutional layers.

To investigate the low-rank structure in the GMFVI posterior of a dense layer, we inspect a spectrum of the posterior mean and standard deviation matrices. In particular, for both the posterior mean and standard deviation matrices, we consider the fraction of the variance explained by the top singular values from their SVD decomposition (see Figure \ref{fig:exp_var} in the main paper). Furthermore, we explore the impact on predictive performance of approximating the reshaped diagonal matrices with their low-rank approximations using only the components corresponding to the top singular values (see Table \ref{tab:lowrank_approx_v1} in the main paper). Note that such low-rank approximations may contain values below zero. This has to be addressed when approximating the matrices of the posterior standard deviations, which can contain only positive values. Therefore, we use a lower bound of zero for the values of the approximations to the posterior standard deviations.

To investigate the low-rank structure in a GMFVI posterior of a convolutional layer, we need to add a few more steps compared to those for a dense layer. In particular, weights of the convolutional layers considered here are 4-dimensional, instead of 2-dimensional as in the dense layer. Therefore, before performing the SVD decomposition, as for the dense layers, we first reshape the 4-dimensional weight tensor from the convolutional layer into a 2-dimensional weight matrix. More precisely, we flatten all dimensions of the weight tensor except for the last dimension (e.g., a weight tensor of shape $[3, 3, 512, 512]$ is reshaped to $[3 \cdot 3 \cdot 512, 512]$). Figure \ref{fig:heatmap_stddev_conv} contains example visualizations of the resulting flattened 2-dimensional matrices\footnote{After this specific reshape operation, all the weights corresponding to a single output filter are contained in a single column of the resulting weight matrix.}.
Given the 2-dimensional form of the weight tensor, we can investigate the low-rank structure in the convolutional layers as for the dense layers. As noted already in Figure \ref{fig:lowrank_approx_resnet_short} in the main paper, we observe the same strong low-rank structure behaviour in the flattened convolutional layers as in the dense layers. Interestingly, the low-rank structure is the most visible in the final convolutional layers, which also contain the highest number of parameters, see Figure \ref{fig:exp_var_conv}. 

Importantly, note that after performing the low-rank approximation in this 2-dimensional space, we can reshape the resulting 2-dimensional low-rank matrices back into the 4-dimensional form of a convolutional layer. Table \ref{tab:lowrank_approx_resnet} shows that such a low-rank approximation of the convolutional layers of the analyzed ResNet-18 model can be performed without a loss in the model's predictive performance, while significantly reducing the total number of model parameters.

\begin{figure*}[ht]
\begin{minipage}[b]{1.0\linewidth}%
\centering%
\includegraphics[width=138mm]{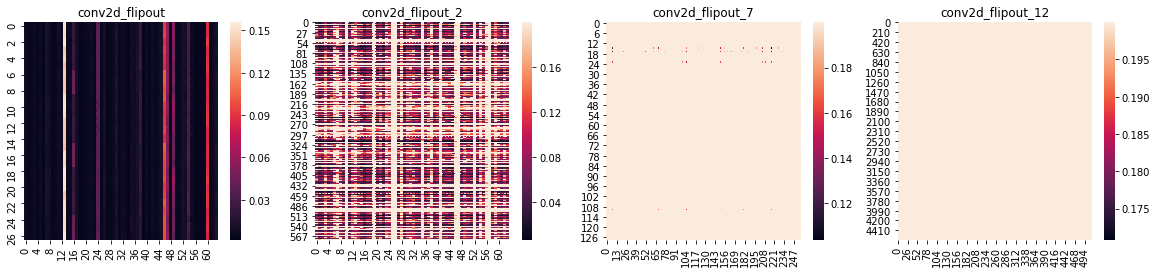}%
\end{minipage}%
\caption{Heat maps of the partially flattened posterior standard deviation tensors for the selected convolutional layers of the ResNet-18 GMFVI BNN trained on CIFAR-10. The partially flattened posterior standard deviation tensors of the convolutional layers display similar low-rank patterns that we observe for the dense layers.}%
\label{fig:heatmap_stddev_conv}%
\end{figure*}

\begin{figure*}[ht]
\begin{minipage}[b]{1.0\linewidth}
\centering
\includegraphics[width=135mm]{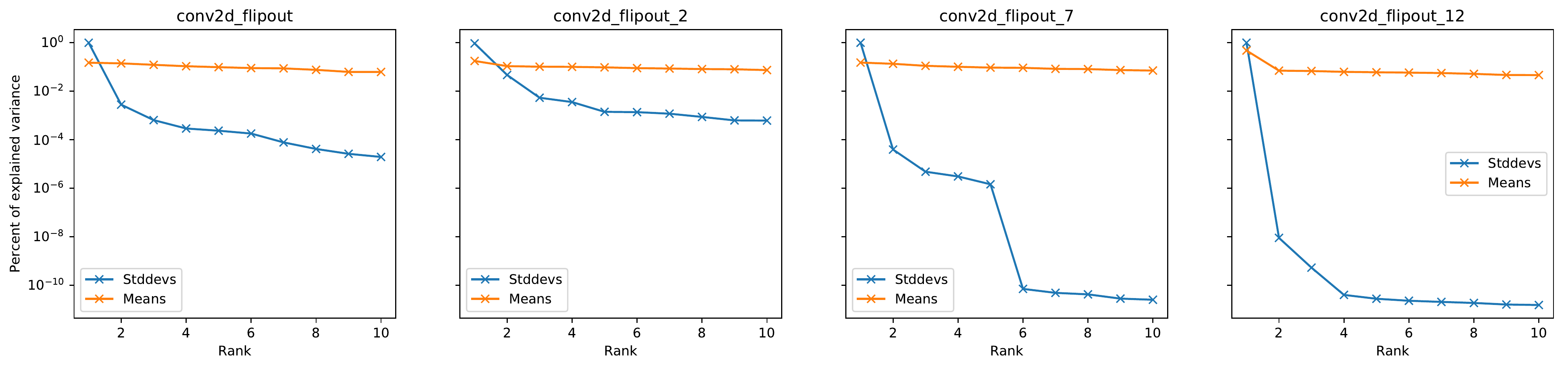}
\end{minipage}%
\caption{Fraction of variance explained per each singular value from SVD of partially flattened tensors of posterior means and posterior standard deviations for different convolutional layers of the ResNet-18 GMFVI BNN trained on CIFAR-10. Posterior standard deviations clearly display strong low-rank structure, with most of the variance contained in the top few singular values, while this is not the case for posterior means. Interestingly, the low-rank structure is the most visible for the final convolutional layers, which also contain the highest number of parameters.} 
\label{fig:exp_var_conv}
\end{figure*}

\begin{table*}[h]
\centering
\begin{tabular}{lccccc}
\hline
Method & -ELBO  $\,\downarrow$ & NLL $\,\downarrow$ & Accuracy $\,\uparrow$ & \#Params $\,\downarrow$ & \%Params $\,\downarrow$ \\ \hline \hline
Mean-field & 122.61\spm{0.012}  & 0.495\spm{0.0080} & 83.5\spm{0.37} & 9,814,026 & 100.0\\ \hline
1-tied    & 122.57\spm{0.012} & 0.658\spm{0.0069} & 81.7\spm{0.39} & 4,929,711 & 50.2\\
2-tied    & 122.77\spm{0.012} & 0.503\spm{0.0080} & 83.2\spm{0.37} & 4,946,964 & 50.4 \\
3-tied    & 122.67\spm{0.012} & 0.501\spm{0.0079} & 83.2\spm{0.37} & 4,964,217 & 50.6\\ \hline
\end{tabular}%
\caption{Impact of the low-rank approximation of the GMFVI-trained posterior standard deviations of a ResNet-18 model on the model's predictive performance. We report mean and SEM of each metric across 100 weights samples. The low-rank approximations with ranks higher than one achieve predictive performance close to that when not using any approximations, while significantly reducing the number of model parameters.}%
\label{tab:lowrank_approx_resnet}%
\end{table*}

\subsection{$k$-tied Normal posterior training}
To exploit the low-rank structure observation, we propose the $k$-tied Normal posterior, as discussed in Section \ref{sec:ktied}. We study the properties of the $k$-tied Normal posterior applied to the MLP, CNN and LSTM models.
We use the $k$-tied Normal variational posterior for all the dense layers of the analyzed models. Namely, we use the $k$-tied Normal variational posterior for all the three layers of the MLP model, for the two dense layers of the CNN model and for the LSTM cell's kernel and recurrent kernel. 

We initialize the parameters $u_{ik}$ and $v_{jk}$ of the $k$-tied Normal distribution so that after the outer-product operation the respective standard deviations $\sigma_{ij}$ have the same mean values as we obtain when using the standard GMFVI posterior parametrization. More precisely, we initialize the parameters $u_{ik}$ and $v_{jk}$ so that after the outer-product operation the respective $\sigma_{ij}$ standard deviations have means at $0.01$ before transforming to log-domain. This means that in the log domain the parameters $u_{ik}$ and $v_{jk}$ are initialized as $0.5 (\log(0.01) - \log(k))$. We also add white noise $\mathcal{N}(0, 0.1)$ to the values of $u_{ik}$ and $v_{jk}$ in the log domain to break symmetry. 

During training of the models with the $k$-tied Normal posterior, we linearly anneal the contribution of the $D_{KL}$ term of the ELBO loss. %
We select the best linear coefficient for the annealing from $\{5 \times 10^{-5}, 5 \times 10^{-6}\}$ (per batch) and increase the effective contribution every 100 batches in a step-wise manner. 
In particular, we anneal the $D_{KL}$ term to obtain the predictive performance results for all the models in Figure~\ref{tab:exploit_res} in the main paper. %
However, we do not perform the annealing in the Signal-to-Noise ratio (SNR) and negative ELBO convergence speed experiments in the same Figure~\ref{tab:exploit_res}. In these two cases, KL annealing would occlude the values of interest, which show the clear impact of the $k$-tied Normal posterior.

\end{document}